\pdfoutput=1

\documentclass[11pt]{article}

\usepackage{acl}

\usepackage{times}
\usepackage{latexsym}
\usepackage[a4paper]{geometry}

\usepackage[T1]{fontenc}

\usepackage[utf8]{inputenc}
\usepackage{multirow}
\usepackage{microtype}

\usepackage{algorithm}
\usepackage{algpseudocode}
\usepackage{amsmath}
\usepackage{booktabs}
\usepackage{amsfonts}
\usepackage{mathtools}
\usepackage{hyperref}
\usepackage{amsthm}
\usepackage{subcaption}
\newtheorem{thm}{Theorem}

\algnewcommand\algorithmicinput{\textbf{Input:}}
\algnewcommand\algorithmicparams{\textbf{Parameters:}}
\algnewcommand\algorithmicoutput{\textbf{Output:}}
\algnewcommand\Input{\item[\algorithmicinput]}%
\algnewcommand\Parameters{\item[\algorithmicparams]}%
\algnewcommand\Output{\item[\algorithmicoutput]}%

\usepackage{amsmath}
\DeclareMathOperator*{\argmax}{arg\,max}

\title{Federated Learning with Noisy User Feedback}
\author{
  Rahul Sharma\thanks{\textit{ equal contributions}} \\
  \texttt{rahul.sharma@usc.edu} \\\And
  Anil Ramakrishna\footnotemark[1]  \\
  \texttt{aniramak@amazon.com} \\\And 
  Ansel MacLaughlin \\
  \texttt{ammaclau@amazon.com} \\\AND
  Anna Rumshisky \\
  \texttt{arum@cs.uml.edu} \\\And
  Jimit Majmudar\\
  \texttt{mjimit@amazon.com} \\\And 
  Clement Chung \\
  \texttt{chungcle@amazon.com} \\\AND
  Salman Avestimehr \\
  \texttt{avestime@usc.edu} \\\And 
  Rahul Gupta \\
  \texttt{gupra@amazon.com}
}

\begin{document}

\maketitle

%

\begin{abstract}
Machine Learning (ML) systems are getting increasingly popular, and drive more and more applications and services in our daily life. This has led to growing concerns over user privacy, since human interaction data typically needs to be transmitted to the cloud in order to train and improve such systems. Federated learning (FL) has recently emerged as a method for training ML models on edge devices using sensitive user data and is seen as a way to mitigate concerns over data privacy. However, since ML models are most commonly trained with label supervision, we need a way to extract labels on edge to make FL viable. In this work, we propose a strategy for training FL models using positive and negative user feedback.
We also design a novel framework to study different noise patterns in user feedback, and explore how well standard noise-robust objectives can help mitigate this noise when training models in a federated setting. 
We evaluate our proposed training setup through detailed experiments on two text classification datasets and analyze the effects of varying levels of user reliability and feedback noise on model performance. We show that our method improves substantially over a self-training baseline, achieving performance closer to models trained with full supervision.
\end{abstract}

\section{Introduction}
\label{sec:intro}
Artificial Intelligence (AI) and Machine Learning (ML) have become increasingly common in modern society with applications ranging from simple standalone products to complex modules \citet{kaplan2019siri}. However, this rise has also created growing privacy concerns \citet{papernot2016towards,yeom2018privacy}.
Such concerns may affect user willingness the adapt new technologies \citet{guhr2020privacy}. In response, many government agencies across the world have introduced regulations to protect the data-handling rights of their citizens, such as the European Union's GDPR \citet{gdpr} and California's CCPA \citet{ccpa}. 

Federated Learning (FL) is a step in this direction to improve consumer trust, where models are trained without moving data out of client devices. 
The typical FL approach is to iteratively train local models on edge devices and then propagate them back to a central node in order to update the global model. Most commonly, this is done using Federated Averaging (FedAvg)~\citet{mcmahan2017communication}, where we take a simple average over the client model parameters.
However, in order to update local models on the edge, this setup assumes the presence of labeled user data on each device, which is often not possible. Most prior works do not address this  problem, but simulate fully-supervised federated learning by distributing existing labeled datasets across edge devices. In this work, we consider a more realistic scenario, where labels are not available on device. Rather than turning to unsupervised learning as seen in \citet{hard2018federated}, we instead propose a novel setup to leverage user feedback in order to train the FL model. 

In many real world AI applications with direct or indirect human interaction, it is possible to collect either explicit user feedback (e.g., using a voice or a screen-based prompt) or implicit feedback in the form of behavioral cues. For an example of implicit feedback, consider a user interacting with a virtual AI assistant (such as Alexa), who asks to play the song ‘Bohemian Rhapsody' from the band Queen. The virtual assistant would interpret the prompt and select a song from its library to play. If the user lets the music play without interruption, this can be viewed as positive feedback, suggesting that the underlying model interpreted the request correctly. On the other hand, if the user interrupts the music and makes a repeat (or different) request, this can be viewed as negative feedback, suggesting that the underlying model prediction was incorrect. In this work, we propose to leverage such feedback in federated model training.

\paragraph*{Leveraging Positive and Negative Feedback} In our proposed setup, we first train a seed model on a small amount of labeled data on a central node. This mimics the real-world scenario where a small amount of data can be collected and annotated to bootstrap an initial model. We then propagate this weak seed model to each of the clients. On the edge, we use this seed model to make predictions for each user's request. Since the model is trained with limited data, these predictions may be incorrect. To further improve this model performance, we leverage user feedback as an indirect indicator of the predicted label's quality. Since positive user feedback likely indicates that a model prediction is correct, we label examples with positive feedback with the seed model's prediction and add them to the training data. This mirrors the standard self-training setup \citet{rosenberg2005semi}, where weak models are further trained on a set of their own predictions. When a user gives negative feedback, however, we cannot assign a label to the example, since we only know that the seed model’s prediction is wrong. We instead treat these prediction as \textit{complementary labels} \citet{ishida2017learning, yu2018learning} and extend federated model training to use such labels. 

\paragraph*{Modeling Feedback Noise} In realistic human interactions, however, the user may not always provide consistent feedback, making user feedback signal noisy. In the virtual AI assistant example above, if the model predicts a different song from the same band, the user may choose to continue listening without interruption. Similarly, even if the model predicts the correct song, the user may change their mind once the song starts playing and interrupt with a new request. Such user behavior will introduce noise into the feedback signal. 
In order to assess typical levels of such noise in user feedback, we conduct a pilot study on Amazon Mechanical Turk (Mturk), and evaluate the accuracy of feedback from Mturk users on two different text classification problems. Based on this study, we define a model of user noise defined by two parameters that specify how likely they are to give accurate feedback on both  correct and incorrect predictions by the seed model. With this model of user behaviour, we then study the effects of user noise on model performance. We conduct extensive experiments on two text classification datasets, training FL models on feedback data with varying amounts of user noise simulated using our model. We further experiment with various noise-robustness strategies to mitigate the effect of such noisy labels and present promising results. 

The key contributions in this paper are as follows: 
\begin{enumerate}
    \item We propose a new framework for modeling and leveraging user feedback in FL and present a strategy to train supervised FL models directly on positive and negative user feedback. We show that, under mild to moderate noise conditions, incorporating feedback improves model performance over self-supervised baselines.
    \item We propose a novel model of user feedback noise and study the effects of varying levels of this noise on model performance. 
    \item We study the effectiveness of existing noise-robustness techniques to mitigate the effects of user-feedback noise and identify promising directions for future exploration.
\end{enumerate}

\section{Related Work}
\label{sec:related work}
\subsection{Federated Learning}
Federated Learning \citet{mcmahan2017communication} has recently seen a rise in popularity in a number of domains, including natural language processing (NLP) ~\citet{yang2018applied, ramaswamy2019federated, hard2018federated}. This is due to growing privacy concerns \citet{papernot2016towards, geyer2017differentially, truex2019hybrid}, abundance of unlabeled data, and an increase in the computational capacity on edge devices. However, availability of labels on edge (or rather, lack thereof ) limits the practical application of FL in most real-world use cases. In this work, we present an extension to FL and show improvements in federated model performance by leveraging user feedback. 
Recent works have also revealed risks of information leakage from gradients in federated learning, and several techniques have been developed to mitigate this risk (see \citet{zhu2019deep}, \citet{lyu2020threats} and the references there in).

\subsection{Learning With User Feedback}
User feedback on model behavior provides learning signals which can be leveraged to continuously improve model performance. Using feedback signals for model training provides robustness to concept and data drifts, as new data is always accompanied with new feedback labels from which to learn. Learning methods that leverage user feedback have been applied to a variety of tasks in NLP, such as semantic parsing \citet{iyer2017learning}, machine translation \citet{kreutzer2018can} and question answering \citet{kratzwald2019learning}. To our knowledge, however, our work is the first to build a parametric model of user feedback noise and to study how to train federated learning algorithms with noisy user feedback. 

\subsection{Negative Label Learning}
Standard supervised learning operates on data labeled with their true classes. Feedback data from users, however, can also be negative, indicating that the predicted class is wrong. Since the correct class of examples with such negative-feedback is unknown, our proposed method must be able to handle such ambiguity during training. In our work, we label examples with negative feedback with a complementary label corresponding to the predicted class \citep{ishida2017learning}. Complementary labels simply specify the category that a given example \textit{does not} belong to. In our work, we follow the setup of \citet{yu2018learning}, who propose loss functions to train neural models on biased complementary labels. 

\subsection{Noise-Robust Learning}
When training models on labels derived from noisy user feedback signals, it is important to use a strategy to mitigate the effects of label noise on model performance. One straightforward approach is to use a noise-robust loss function, such as Reverse Cross Entropy \citet{wang2019symmetric} or Mean Absolute Error~\citet{Ghosh2017RobustLF}.
In our work we follow the noise-robust learning setup of \citet{ma2020normalized}, who present a training strategy that combines two robust loss functions (one active, one passive) to better handle label noise. 



\section{Modeling User Feedback}

We propose a general framework for leveraging user feedback in federated learning. We use text classification as an exemplar task to evaluate this framework, but the proposed setup can be easily applied to other tasks. We use two benchmark text classification datasets: the Stanford Sentiment Treebank dataset (\emph{sst2}) and the 20 newsgroups dataset (\emph{20news}). The \emph{sst2} dataset comprises of $11,855$ phrases from movie reviews and the corresponding binary sentiment labels. The 20 newsgroup dataset \emph{20news} consists of 
$18k$ news articles and headers, organized into 20 classes.

\subsection{Pilot Study: Real World User Feedback}
\label{sec:case-study}
To understand the dynamics of user feedback noise, we conduct a pilot study using Amazon Mechanical Turk (Mturk). We use text classification on the above two datasets, \emph{sst\_2} and \emph{20news}, as the task of interest. For each dataset, we train a seed model on $1\%$ of the training data, then run inference with this model to generate pseudo-labels on the remaining $99\%$ of the training examples. 
We sample 2000 pseudo-labeled examples from this set, split them into disjoint groups of 40 examples each, and show them to $50$ and $40$ different MTurk workers for \emph{sst\_2} and \emph{20news}, respectively. For each example, the worker is shown the text of the example and the corresponding predicted pseudo label.
The workers are then asked to specify whether the pseudo label is accurate (positive feedback) or not (negative feedback) along with an option to mark `I Don't Know' in case they find it difficult to decide. Further details about the specific instructions used in our Mturk study can be found in Appendix \ref{appendix:mturk_study}.
We use the ground truth gold labels in \emph{sst\_2} and \emph{20news} to evaluate the quality of user feedback by computing the proportions of times users gave positive feedback to correct pseudo labels and negative feedback to incorrect ones. The observed average error in feedback is $33.9\%$ for 20news and $27.13\%$ for sst2. The higher observed error for \emph{20news} is likely due to the fact that 20news is a 20-way topic classification problem with overlapping categories such as `autos' and `motorcycles'. We further analyze the collected data using the noisy feedback model described next. Full data will be released with the final draft of the paper. 

\subsection{Feedback Noise Modeling}
\label{subsec:user_feedback}
Motivated by the observed noisy user behavior above, we propose a parametric noise model using two \textit{user-specific} Bernoulli random variables parameterized by $\gamma$ and $\delta$, as shown below.
\begin{gather*}
    P(\text{positive feedback} | \text{correct prediction}) = \gamma \\
    P(\text{negative feedback} | \text{incorrect prediction}) = \delta
\end{gather*}
$\gamma$ and $\delta$ capture the probability of the user accurately providing positive and negative feedback, respectively. This scheme provides a powerful tool to model user noises of various types - by varying the values of these two parameters, we can simulate various user feedback behaviors. For instance, highly reliable users can be simulated by choosing $\gamma \sim 1$ and $\delta \sim 1$ while adversarial users can be simulated by choosing $\gamma \sim 0$ and $\delta \sim 0$. Similarly, users that provide consistently positive feedback can be simulated by selecting $\gamma \sim 1$ and $\delta \sim 0$, and vice versa.

\begin{figure}[tbh]
    \centering
    \includegraphics[width=0.45\textwidth,keepaspectratio]{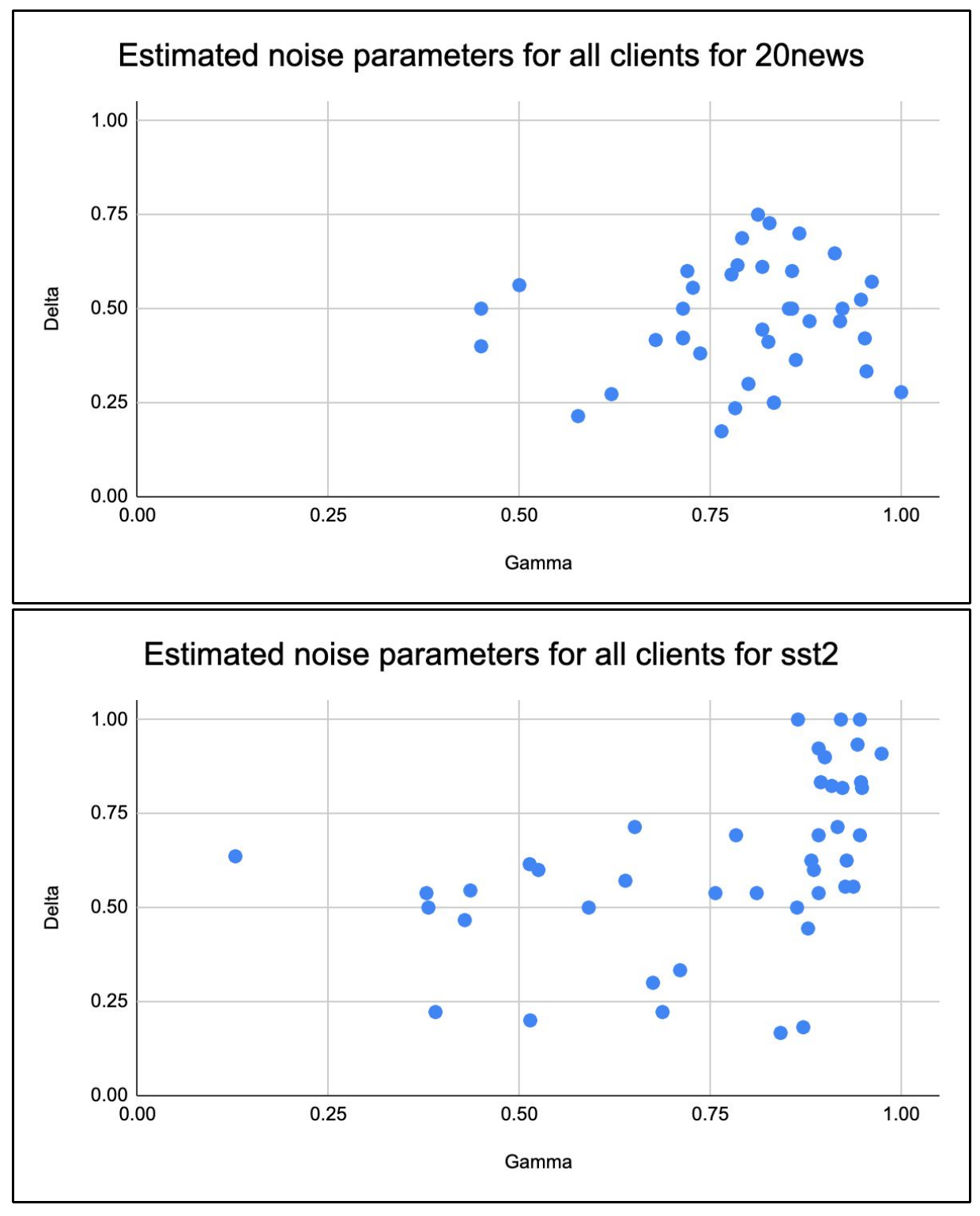}
    \caption{Distribution of noise parameters $\gamma$ and $\delta$ for annotators on Mturk for \emph{20news} and \emph{sst2} dataset.}
    \label{fig:mturk_client_wise}
\end{figure}

For each worker in our MTurk study, we estimate the noise parameters $\gamma$ and $\delta$ using the MLE estimators described in Appendix \ref{appendix:estimators}. We show the distributions over the estimated noise parameters for each worker in Figure~\ref{fig:mturk_client_wise}, which highlights several characteristics of the user noise. We find that parameters vary across users and across datasets. Many workers have high values for both $\gamma$ and $\delta$ (upper right quadrant in the plot), especially for the \emph{sst2} dataset. In such cases, user noise is relatively low. Some workers have $\gamma \sim 1$ and $\delta \sim 0$, suggesting that they provide positive feedback with very high probability, regardless of the correctness of the pseudo label. Similarly, we also observe some points with higher $\delta$ but $\gamma$ close to $0$, suggesting that these workers provide negative feedback with high probability. Since we only recruited reliable worker from Mturk ($95\%+$ approval rate and $5000+$ approved HITs, see Appendix \ref{appendix:mturk_study}), we do not see any workers in the adversarial or extremely-high noise scenarios (lower-left quadrant in the plot). Finally, we also observe that the workers in the \emph{sst2} dataset are more concentrated towards the right top corner while, for the \emph{20news} dataset, they are relatively spread out. This can be attributed to the inherent difficulty of the two datasets -- \emph{sst2} is an easier 2-class sentiment classification dataset, while \emph{20news} is a more difficult news-classification dataset with 20 (sometimes overlapping) classes. 

\section{Approach}
\subsection{Federated Self-Training}
\label{subsec:federatedSelfTraining}
Given a training dataset $D_{\text{t}} = \{x_i, y_i\}$, we divide it into 3 parts: a training split $D_{\text{s}} \subset D_{\text{t}}: |D_{\text{s}}| = k|D_{\text{t}}|, k\ll 1$, used to train the seed model; a validation split $D_{\text{v}} \subset D_{\text{t}}: |D_{\text{v}}| = v|D_{\text{t}}|, v < 1$ and an unlabeled split $D_{u} = D_{\text{t}} - (D_{\text{s}} \bigcup D_{\text{v}})$.   
We assume that the examples in $D_{\text{s}}$ and $D_{\text{v}}$ have gold labels available for training, which mimics the real-world situation where a small amount of data can be annotated to bootstrap the model training. We treat $D_{\text{u}}$ as the unlabeled dataset which is available on edge. We initialize the seed model $f_{\text{s}}(x)$ by training on $D_{s}$ using standard cross entropy loss.
After convergence, this model, $f_{\text{s}}(x)$, is deployed to the edge devices to start federated training. In order to simulate a real-world federated learning setup, we distribute the examples from $D_{\text{u}}$ among $N$ edge clients following a skewed distribution, described in detail in \S\ref{sec:Exp}. The dataset on each client $n$ is labeled $D_{u}^{n}$ where $n \in [1,N]$. The seed model on device $j$ then makes predictions on its corresponding client-specific dataset $D_{u}^{j}$. 
Since the edge model does not have access to gold labels for training, there are only two potential sources of information it can learn from. First, its own predictions, $\rho_{i}$, which we call pseudo labels:
\begin{align}
    \rho_{i} = \argmax(f_{\text{s}}(x_i)): i \in D_{u}^{n}
\end{align}
Labeling an example $x_i \in D_{u}^{n}$ with $\rho_{i}$ is typically referred to as self-training, a commonly-used semi-supervised training setup. However, in our setup, there is also a second source of information: user feedback to each $\rho_{i}$. We assume that users give binary (positive or negative) feedback to each $\rho_{i}$. We can thus use this feedback to validate or reject each $\rho_{i}$, generating label $\rho_{i}$ when the feedback is positive and $\bar{\rho_{i}}$ when the feedback is negative. Then, with these new user-feedback-labeled datasets on each edge device, we can follow the standard FL training, further training a copy of the initial global model on each edge device, then propagating each local model back to the global server for aggregation. Our overall setup used for federated learning with user feedback is shown in Figure \ref{fig:overall_arch}. 

%

\begin{figure}[tbh]
\includegraphics[width=0.45\textwidth,keepaspectratio]{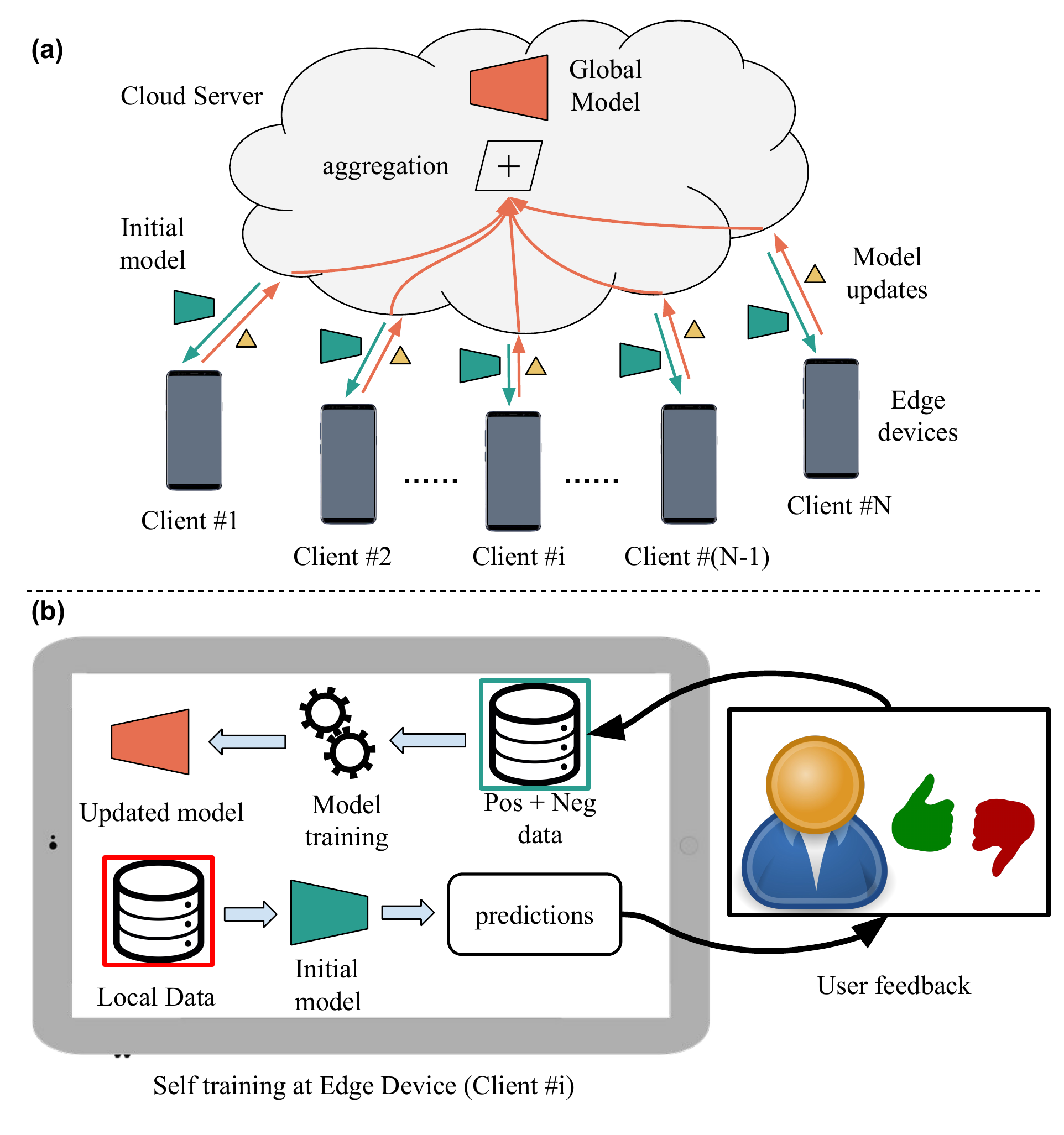}
    \caption{Overview of our federated learning setup with user feedback. a) Federated learning with a central cloud server and several client devices. b) Local training at a particular client with user feedback to improve pseudo labels.}
    \label{fig:overall_arch}
\end{figure}

\subsection{User Feedback Simulation}
\label{subsec:feedback_sim}

As discussed in \S\ref{sec:intro}, in the real world, users provide feedback on predictions made by  deployed models. However, large-scale collection of user feedback in a deployed FL application is an expensive endeavor with no publicly-available datasets. In this work, we instead devise a novel framework to simulate realistic noisy user feedback on publicly-available, human-annotated data, and defer the task of real world deployment to future work. Specifically, when we distribute the unlabeled dataset $D_{u}$ across the $N$ client devices, we also send along the true gold label for each example $x_i$. For each $x_i \in D_{u}^{n}$, we then simulate feedback by comparing the model prediction $\rho_{i}$ to its underlying gold label. These gold labels are only used to simulate user feedback -- they are not used for edge model training. Specifically, we create two new pseudo-labeled datasets corresponding to positive ($D_{pos}^{n}$) and negative feedback ($D_{neg}^{n}$) from each device's dataset $D_{u}^{n}$. For a given sample $x_i \in D_{u}^{n}$, we assign it to the positive feedback set $D_{pos}^{n}$ with probability $\gamma$ if the corresponding pseudo label $\rho_{i}$ is correct and $1-\delta$ if $\rho_{i}$ is incorrect. Similarly, we assign a sample to the negative feedback set $D_{neg}^{n}$ with probability $1-\gamma$, if $\rho_{i}$ is correct and $\delta$ if $\rho_{i}$ is incorrect. A pseudo-code detailing our strategy to simulate user feedback is shown in Algorithm ~\ref{alg:userfeedback}.

\subsection{Federated Learning with Feedback}
\label{subsec:nll}
For examples with positive user feedback, since we have user confirmation that the pseudo-label $\rho_{i}$ is correct, we directly incorporate $\rho_{i}$ into model training as if they were ground-truth. We use the standard categorical cross entropy (CCE) loss function similar to the seed model:
\begin{align}
\label{eqn:loss_poss}
    loss_{\text{pos}} = -\sum_{i \in D_{pos}^{n} }\rho_{i}log(f(x_i)) 
\end{align}
where $f(x_i)$ represents the posterior probability distribution for sample $i$ and $\rho_i$ is overloaded to depict the one-hot representation  of the pseudo label for sample $i$.
On the other hand, negative feedback signifies that the sample does not belong to the class $\rho_{i}$. Although the user feedback does not indicate which class these samples ought to be, we do acquire information for what the model \textit{should not do}. Thus we can assume that these are complementary labels, denoted as $\bar{y_{i}} = \rho_{i}$. 
To incorporate these in our model training, we adapt the complementary learning methods introduced by \citet{yu2018learning}, in which the authors model the complementary posterior probability distribution, $P(\bar{Y} = d | X)$ as a function of true class posterior distribution, $P(Y=c|X)$ and the transition probability matrix $\mathbf{Q}$, where $q_{cd}$ is an entry in the matrix $\mathbf{Q}$ with $c$ and $d$ are class labels, following:

\begin{align}
    q_{cd} &=
    \begin{cases}
    P(\bar{Y} = d | Y= c)  &c \neq d\\
    0  &c = d
    \end{cases} \label{eq:complementary_prob} \\
    P(\bar{Y} = d | X) &= \sum_{c \neq d}P(\bar{Y} = d | Y= c) \nonumber \\  &P(Y=c|X)
\end{align}

We estimate the transition probability matrix $\mathbf{Q}$ using the validation set $D_{v}$ and the initial seed model $f_{\text{s}}(x)$. To compute $\mathbf{Q}_{c:}$, the transition probability distribution for the class $c$, we average the posterior probability of those samples with gold label $c$, but are incorrectly predicted by the model. Given this, we set $q_{cc} = 0$ and renormalize the vector as shown in Equation~\ref{eq:Q_matrix}.
\begin{align}
    \label{eq:Q_matrix}
    \mathbf{Q}_{c:} &= \frac{1}{K}\sum_{k \in D_{c}}\frac{f_{s}(x_{k})}{1 - f_{s}(x_k)_i}: K = |D_{c}| \\
    D_{c} &\subset D_{v}: argmax(f(x_k)) \neq c; y_k = c \nonumber
\end{align}
Using the estimated transition matrix, and the posterior distribution for the true class, we estimate the posterior distribution for the complementary class, following Equation~\ref{eq:complementary_prob}. We then use the pseudo labels as complementary labels and train the model with the objective function:
\begin{align}
\label{eqn:loss_neg}
loss_{\text{neg}} = -\sum_{i \in D_{neg}^{n}}\rho_ilog(\mathbf{Q}.f(x_i)) 
\end{align}
\noindent here, we overload $\rho_i$ to depict the one-hot representation of the pseudo label for sample $i$.

Our overall model is trained to jointly optimize the loss functions from positive feedback and negative feedback. Inspired by \citet{kim2020pseudo} where the contribution of negative learning is slowly increased during training, we use a scheduler to weigh the two loss functions, ensuring that the positive label learning component has higher weight in the early epochs, gradually increasing the weight for negative label learning as training proceeds. Specifically, at each client we optimize the following objective:
\begin{align}
    loss_{reg} &= (1 - \alpha)*loss_{pos} + \alpha *loss_{neg} \label{eq:apl-loss_3}
\end{align}
where, $\alpha = 1 - p^{t}$, $t$ denotes the current epoch in the training process and $p \in (0,1)$, which was selected using a held out validation set. 

\subsection{Noise-Robust Loss Functions}
\label{subsec:noise-robust}
Though user feedback provides a valuable learning signal to train our models on edge, it can be noisy. As noted in \S\ref{subsec:user_feedback}, we expect two behaviors from noisy users: if the user provides incorrect positive feedback, we have incorrect true labels in $D_{pos}^{n}$. Similarly, if the user provides incorrect negative feedback, we have samples in the $D_{neg}^{n}$ with incorrect complementary labels. Since we use cross entropy loss functions for training on both positive and negatively labeled data points, our model is prone to overfitting to noisy samples \citet{zhang2018generalized} since they would have lower posteriors (forcing the algorithm to put more emphasis on them during training). This necessitates some form of noise mitigation in our model training. 

To mitigate label noise, we use the approach introduced by \citet{ma2020normalized}, where they propose to add noise robustness to any given loss function by normalizing it across all labels. \citet{ma2020normalized} further improve convergence by presenting a combination loss function with active and passive loss components, to maximize the posterior for the true class and to minimize the posterior for complementary classes, respectively. In our experiments, we use a combination of Normalized Cross Entropy (NCE) \citet{ma2020normalized} and Reverse Cross Entropy (RCE) \citet{wang2019symmetric} as the active and passive components, weighted in a ratio 1:2 selected using our validation set.
\begin{equation}
    \label{eq:noise_robust_loss}
    loss_{robust} = NCE + 2*RCE
\end{equation}

The NCE and RCE functions are listed below. 
\begin{align}
    \label{eq:apl-loss}
    NCE &= \frac{-\sum_{k=1}^{K}q(k|x)\log p(k|x)}{-\sum_{j=1}^{K}\sum_{k=1}^{K}q(y=j|x)\log p(k|x)}\\
    RCE &= -\sum_{k=1}^{K}p(k|x)\log q(k|x))
\end{align}
where $K$ is number of label classes, $q(k|x)$ denotes the gold label distribution and $p(k|x)$ denotes the posterior probability distribution.
\section{Experiments}
\label{sec:Exp}

\subsection{Implementation Details}
We use the publicly-provided train and test splits for the \emph{sst2} and \emph{20news} datasets and further derive a validation split consisting of 20\% ($v=0.2$) of the train split ($D_t$), with uniform class distribution. We use another 1\% ($k=0.01$) of $D_t$ as seed model set, $D_s$. We choose a small value for $k$ to mimic a real world scenario where a larger volume of data may be un-annotated in an FL setup.
The remaining unlabeled dataset $D_u$ (79\% of the ($D_t$) is further divided among 15 mutually exclusive subsets ($D_{u}^{n}, n \in [1, 15]$), which simulates the data for 15 edge clients. While creating the clients' partitions we ensure that all clients have data with a uniform class distribution which enables us to focus on our model performance in an idealized case. 
We use the DistilBERT~\citet{sanh2019distilbert} base model as the classifier for our tasks and follow the standard fine-tuning setup for text classification. To implement the federated learning pipeline we use the publicly-available FedNLP~\citet{fednlp2021} benchmark and apply the FedAvg~\citet{mcmahan2017communication} algorithm to aggregate the client model updates at the server end. We train the model on an 8-GPU (Nvidia V100s) machine for up to 50 rounds with early stopping enabled. Within each round, we use a batch size of 8 to train the client models for 5 epochs each. 

\begin{table}[tb]
    \centering
    \setlength\tabcolsep{2.5pt}
    \begin{tabular}{@{}c|c|c@{}}
    \textbf{Experimental settings}          & \textbf{20 news} & \textbf{sst2}  \\ \hline
    Initial model ($D_s$)                & 59.14   & 77.37 \\ 
    Self training (no feedback)                          & 60.79   & 77.26 \\ \hline
    Positive feedback (noisy)   & 62.10   & 79.79 \\ 
    All feedback (noisy)        & 65.01   & 85.17 \\ \hline
    Positive feedback (noise robust) &62.33 &79.85 \\ 
    All feedback (noise robust)            &65.13 &85.39 \\ \hline
    Positive feedback (noise free) & 70.44   & 83.80 \\ 
    All feedback (noise free)       & 75.13   & 88.58 \\ \hline
    Full supervision                          & 82.12   & 89.12 \\
    \end{tabular}
    \caption{Accuracy of noise robust federated self training with user feedback against various baselines for \emph{20news} and \emph{sst2} datasets; *: all models using feedback (with and without noise robustness) are statistically significant against the self training baseline (without feedback), at $p<0.05$.}
    \label{tab:user_feedback_eval}
\end{table}


\subsection{Model Evaluation}
\label{sec:model-evaluations}

We compare our model performance against three baselines: 

\textbf{\textit{Initial model}} This is the seed model $f_{\text{s}}(x)$ trained on just $D_s$ ($1 \%$ of the training data).

\textbf{\textit{Self-training}} We train this model using federated learning with pseudo labels, but \textit{do not} utilize the user feedback. Hence, at each client, we only have the raw pseudo labels $\rho_{i}$ for each $x_i \in D_{u}^{n}$ to train on. 
We use this setup as the primary baseline against which to compare the performance of our models trained with user feedback. 

\textbf{\textit{Full supervision}} We train a model in a federated setting using $D_{u}^{n}$ and the true gold labels at each client. Although in a real-world setting, the clients will not have access to the gold labels, we establish this benchmark to set an upper bound.

We evaluate performance of our proposed strategy of leveraging user feedback in two settings: 

\textbf{\textit{Positive feedback}} At each client, we train the local version of the model using only the samples in $D_{pos}^{n}$ and corresponding pseudo labels $\rho_{i}$, i.e. only the samples which receive positive feedback. Since this baseline is trained using regular cross entropy, it provides a powerful yet computationally less-intensive alternative to training with both types of feedback, which is especially important in edge devices with low compute power. 

\textbf{\textit{All feedback}} We utilize all the data samples in $D_{pos}^{n}$ and $D_{neg}^{n}$ at each client and train using the loss function described in Section~\ref{subsec:nll}. 

In both the \textit{positive} and \textit{all feedback} setups, we evaluate models with and without user feedback noise. For the noise-free scenario, we set $\gamma=1$ and $\delta=1$ while simulating the user feedback. This leads to perfectly accurate feedback, as discussed in \S\ref{subsec:feedback_sim}.
For the noisy feedback scenario, we use the noise parameters derived from the Mturk study. We obtain the following dataset-specific values of $\gamma$ and $\delta$ by averaging the estimates of $\gamma$ and $\delta$ across all annotators: ($\gamma=0.79$, $\delta=0.55$) for \emph{20news} and ($\gamma=0.76$, $\delta=0.70$) for \emph{sst2}.

Table~\ref{tab:user_feedback_eval} reports the \% accuracy for each of the experimental setups described above across both datasets. We observe that in both the noisy and noise-free settings, the introduction of positive user feedback shows a marked improvement in performance when compared to the self-training baseline. There is an additional performance gain when we add negative feedback (all feedback baseline), which signifies the importance of learning from complementary labels. As expected, the improvement is substantially larger in the noise free setting, suggesting the need for model robustness to mitigate the effect of noise. Note that for $sst2$, performance of the noise free model with all feedback is very close to that of full supervision, thanks to the fact that complementary labels in the case of binary classification provide same information as true labels. On the other hand, using perfect positive and negative feedback in $20news$ is still sub-optimal compared to full supervision, since a negative label in this dataset is less informative compared to $sst2$. 
\subsection{Noise Robustness}
To mitigate the effects of noise, we replace the traditional cross-entropy loss function with the active-passive loss described in \S\ref{subsec:noise-robust}, using the same experimental setups presented earlier (positive only and all-feedback), with $\gamma$ and $\delta$ values from the Mturk study. However, as evident in Table~\ref{tab:user_feedback_eval}, the robust loss functions only seem to confer marginal performance improvements in both datasets. This is likely due to the fact that the noise parameters extracted from Mturk belong to a moderate to low noise regime (Section~\ref{subsec:user_feedback}), providing limited room for gains with noise robustness. 

To further investigate this, we explore two extreme cases of user feedback noise for the \emph{20news} dataset: i) low noise, where we simulate user feedback with $\gamma \rightarrow 1$, $\delta \rightarrow 1$ for all the clients, which imitates clients providing correct feedback with very high probability, and 
ii) adversarial noise, with $\gamma \rightarrow 0$, $\delta \rightarrow 0$ for all the clients, which captures the possible risk of users deliberately providing incorrect feedback with high probability. In Table~\ref{tab:noise_robust_high_noise}, we compare the performances of the all feedback model trained with and without noise robustness in these two scenarios. As seen in the table, when user noise is high, the noise-robust loss functions show a statistically significant improvement against the noisy model, highlighting the value of adding noise robustness. In the low noise regime, adding robustness seems to cause negligible degradation in accuracy, but within the bounds of statistical error. Given this, we recommend using noise robustness in all applications of this framework unless it is well known before hand that the feedback has very low noise. We defer the task of developing a noise robustness regime that works for all noise levels to future work.     

    
\begin{table}[tb]
    \centering
    \begin{tabular}{c|c|c}
    \textbf{Noise level} & \textbf{Loss} & \textbf{Accuracy} \\ \hline
    \multirow{2}{*}{Low} & $loss_{robust}$         &73.29   \\
    & $loss_{reg}$ &74.30   \\ \hline
    
    \multirow{2}{*}{Adversarial} & $loss_{robust}$        &42.26*\\
    & $loss_{reg}$                       &25.19   \\ 
    \end{tabular}
    \caption{Performance analysis of noise robust loss functions trained on all feedback in different noise regimes for the \emph{20news} dataset; *: statistically significant against the adversarial model without robustness at $p<0.05$.}
    \label{tab:noise_robust_high_noise}
    \end{table}
\subsection{Ablation Studies}

\subsubsection{Varying Degrees of Noise}
As discussed in \S\ref{subsec:noise-robust}, the level of feedback noise has a substantial impact on model performance. 
In this section, we further investigate this effect, simulating user feedback across various noise parameters values, spanning $\gamma, \delta \in \{0.3,0.5,0.7\}$, to capture different points in the $\gamma - \delta$ space. Table~\ref{tab:varying_noise} shows our results for each dataset with the noise robust loss function \ref{eq:noise_robust_loss}. As expected, as $\gamma \rightarrow 0$ and/or $\delta \rightarrow 0$, model performance decreases on both datasets. At very low values of $\delta$ and $\gamma$, e.g. both $\leq$ 0.5, training on the extremely noisy user feedback actually decreases model performance below the original seed model. This is not unexpected, since at $\delta = 0.5$ and $\gamma = 0.5$, user feedback is essentially random noise, and at lower values the feedback is adversarial. These results highlight the importance of evaluating the reliability of user feedback before using it to further train an ML system.


\begin{table}
    \centering
    \begin{subtable}[h]{0.45\textwidth}
    \centering
    \begin{tabular}{@{}c|c|c|c@{}}
    $\gamma/\delta$ & 0.7   & 0.5   & 0.3 \\ \hline
    0.7         & 66.69 & 63.18 & 60.66 \\ 
    0.5         & 65.56 & 59.15 & 59.73 \\ 
    0.3         & 60.01 & 58.94 & 58.21 \\ 
    \end{tabular}
    \caption{\emph{20news} dataset; initial model performance: 59.14, performance with all feedback and no noise ($\gamma = \delta = 1$): 75.13.}
    \end{subtable}
    
    \begin{subtable}[h]{0.45\textwidth}
    \centering
    \begin{tabular}{@{}c|c|c|cc@{}}
    $\gamma/\delta$ & 0.7   & 0.5   & 0.3 \\ \hline
    0.7         & 83.86 & 80.89 & 76.17 & \\ 
    0.5         & 81.99 & 77.38 & 75.07 \\ 
    0.3         & 78.03 & 74.41 & 71.99 \\ 
    \end{tabular}
    \caption{\emph{sst2} dataset; initial model performance: 77.37, performance with all feedback and no noise ($\gamma = \delta = 1$): 88.58.}
    \end{subtable}
    \caption{Model performance (accuracy) at varying values of $\gamma$ and $\delta$}
    \label{tab:varying_noise}
    \end{table}

\subsubsection{Non-identical User Feedback Behavior}
\label{subsec:client-dependent}
In previous sections, we used identical values of the noise parameters $\gamma$ and $\delta$ for all clients in the FL training setup. However, as observed in our Mturk study, real users have different feedback behaviors, with scores spread out over the $\gamma - \delta$ space. In this section, we simulate non-identical user feedback for four potential user behaviors: 
\begin{enumerate}
    \item Low noise users ($\gamma \rightarrow 1, \delta \rightarrow 1$)
    \item Adversarial/high noise users ($\gamma \rightarrow 0, \delta \rightarrow 0$)
    \item Positive users ($\gamma \rightarrow 1, \delta \rightarrow 0$) who provide consistently positive feedback, regardless of the correctness of the model prediction; and 
    \item Negative users ($\gamma \rightarrow 0, \delta \rightarrow 1$) who provide consistently negative feedback. 
\end{enumerate}

To simulate non-identical user feedback for the clients, we sample the noise parameters from an appropriately skewed $\beta(a,b)$ distribution. 
For example, in order to generate $\delta$ and $\gamma$ scores for setup four (negative users), each user needs $\gamma \approx 0$, $\delta \approx 1$. To generate these parameters, we sample $\gamma$ from $\beta(1, 10)$ and $\delta$ from $\beta(10,1)$. Proceeding this way, we can simulate all four user behaviors listed above. Finally, we also conduct an experiment closer to the real-world scenario, where we randomly sample 15 workers each for both datasets from our Mturk study and use their estimated values of $\gamma$ and $\delta$ to simulate user feedback. 

Table~\ref{tab:noise dependent on user} shows our results across all five simulations for both datasets when trained with the noise robust loss function \ref{eq:noise_robust_loss}. As expected, the best model performance is achieved with the low-noise users, followed by the real-world users sampled from our MTurk study. In the three other simulations (adversarial, consistently positive, consistently negative), user feedback is highly noisy and unreliable, and the models show limited improvement over the initial seed model. Note that the performance in the positive feedback scenario is higher than negative feedback, which can be accredited to the fact that the initial seed model's accuracy is greater 50\% for both datasets (Table~\ref{tab:user_feedback_eval}). With $>50\%$ accuracy, a majority of the pseudo-labels generated using the seed model will match the gold label. Hence, consistently positive feedback introduces less noise and in turn better performance compared to the all negative feedback model.



\begin{table}
    \centering
    \begin{tabular}{@{}c|c|c@{}}
    \textbf{User Behavior}            & \textbf{20news}                     & \textbf{sst2}  \\ \hline
    Low noise         & \multicolumn{1}{c|}{73.67} & 88.35 \\
    Adversarial          & \multicolumn{1}{c|}{55.86} & 64.85 \\
    Always positive & \multicolumn{1}{c|}{60.99} & 77.16 \\
    Always negative & \multicolumn{1}{c|}{58.92} & 74.13 \\
    Real world (mturk study)        & \multicolumn{1}{c|}{65.37} & 85.61 \\ 
    \end{tabular}
    \caption{Model performance at various user behaviors.}
    \label{tab:noise dependent on user}
    \end{table}

\section{Conclusion}
In this work, we propose a novel framework for federated learning which leverages noisy user feedback on the edge. Our framework eliminates the need for labeling edge data, in turn improving customer privacy since we no longer need to move data off of edge devices for annotation. In order to evaluate our framework, we propose a method to simulate user feedback on publicly-available datasets and a parametric model to simulate user noise in that feedback. Using that method, we conduct experiments on two benchmark text classification datasets and show that models trained with our framework significantly improve over self-training baselines. 

Future work includes deploying our framework in a real world FL application and incorporating live user feedback in model training. We can also explore other noise-robustness strategies for low and medium label-noise scenarios. One such strategy would be to incorporate a measure of user reliability into the calculation of each new global model in FedAVG -- e.g. the updated global model parameters could be computed as a weighted average of client models, with the weight capturing some measure of each client's reliability. 

\section{Ethics Statement}
Our Mturk data collection recruited annotators from across the globe without any constraints on user demographics. The annotators were compensated with above minimum wages and no personal information was collected from the annotators. 

\bibliography{ref}

\appendix 

\section{Pseudocode}
Algorithm ~\ref{alg:userfeedback} lists our training loop. 

\begin{algorithm*}[tbh]
    \caption{Algorithm for simulating user feedback}
    \label{alg:userfeedback}
    \begin{flushleft}
        \textbf{INPUT:} Client data: $D_{u}^{n} = \{x_i, y_i\}$; Pseudo labels: $\rho_{i}$ \\
        \textbf{OUTPUT:}  $D_{pos}^{n}$ and $D_{neg}^{n}$
    \end{flushleft}
    \begin{algorithmic}[1]
        \State $D_{pos}^{n} \gets \{\}$, $D_{neg}^{n} \gets \{\}$
        \For{sample i in $D_{u}^{n}$}
            \If{$y_i == \rho_i$} \Comment{correct model prediction}
            \State $D_{pos}^{n} \gets \{D_{pos}^{n} \cup i\}$ with probability $\gamma$
            
            \textit{or} 
            \State $D_{neg}^{n} \gets \{D_{neg}^{n} \cup i\}$ with probability $1 - \gamma$ \Comment{noise}
            \ElsIf{$y_i != \rho_i$} \Comment{incorrect model prediction}
            \State $D_{neg}^{n} \gets \{D_{neg}^{n} \cup i\}$ with probability $\delta$
            
            \textit{or}
            \State $D_{pos}^{n} \gets \{D_{pos}^{n} \cup i\}$ with probability $1 - \delta$ \Comment{noise}
            \EndIf
        \EndFor
        \State \textbf{return} $D_{pos}^{n}$ and $D_{neg}^{n}$
  \end{algorithmic}
\end{algorithm*}

\section{Estimators for \texorpdfstring{$\gamma$}{g} and \texorpdfstring{$\delta$}{d}}
\label{appendix:estimators}
Let $X$ be the data set and $n$ be the total number of data points. For any data point $i \in [n]$, let $p_i$, $t_i$ and $f_i$ denote the model predicted label, ground truth label, and user feedback respectively. (Note that $p_i$ and $t_i$ take values from the set of labels and $f_i$ takes values from the set $\{ pos, neg, idk  \}$ representing feedbacks positive, negative, and 'I don't know'.) By definition, we have
\begin{align*}
	\gamma &:= \Pr(f_i = pos \mid p_i=t_i)  \\
	\delta &:= \Pr(f_i = neg \mid p_i \neq t_i)
\end{align*}

Let us also define 
\begin{align*}
	\alpha &:=  \Pr(f_i = neg \mid p_i=t_i)  \\
	\beta &:= \Pr(f_i = pos \mid p_i \neq t_i)
\end{align*}

Note that the above definitions imply that
\begin{align*}
	1-\alpha-\gamma &=  \Pr(f_i = idk \mid p_i=t_i)  \\
	1-\beta-\delta &= \Pr(f_i = idk \mid p_i \neq t_i)
\end{align*}

Moreover, let $a$ denote the accuracy of the labels predicted by the model defined as 
\begin{align*}
	 a &:= \Pr(p_i=t_i)
\end{align*}

Define sets $\{S_j\}_{j \in [6]}$ such that
\begin{align*}
	S_1 &:=  \{ i \in [n]:  f_i = pos \text{ and } p_i=t_i \} \\
	S_2 &:=  \{ i \in [n]:  f_i = pos \text{ and } p_i \neq t_i \} \\  
	S_3 &:=  \{ i \in [n]:  f_i = neg \text{ and } p_i = t_i \} \\  
	S_4 &:=  \{ i \in [n]:  f_i = neg \text{ and } p_i \neq t_i \} \\  
	S_5 &:=  \{ i \in [n]:  f_i = idk \text{ and } p_i = t_i \} \\  
	S_6 &:=  \{ i \in [n]:  f_i = idk \text{ and } p_i \neq t_i \}
\end{align*}
define $n_j := \lvert S_j \rvert$. Note that $\sum\limits_{j \in [6]} n_j = n$.

\begin{thm}
	The maximum likelihood estimators for $\gamma$ and $\delta$ are $n_1/(n_1 + n_3 + n_5)$ and $n_4/(n_2 + n_4 + n_6)$ respectively. 
\end{thm}
\begin{proof}
Now for any data point $i \in [n]$, we have
\begin{align*}
	\Pr(i \in S_1) &= \Pr(f_i = pos \text{ and } p_i=t_i) \\
	&=   \Pr(f_i = pos \mid p_i=t_i) \cdot \Pr(p_i=t_i) \\
	&= \gamma  a.
\end{align*}

By a similar reasoning, we have
\begin{align*}
	\Pr(i \in S_2) &= \beta (1-a) \\
	\Pr(i \in S_3) &= \alpha a \\
	\Pr(i \in S_4) &= \delta (1-a) \\
	\Pr(i \in S_5) &= (1-\alpha-\gamma) a \\
	\Pr(i \in S_6) &= (1-\beta-\delta) (1-a)
\end{align*}

Therefore the likelihood function of the model is
\begin{equation*}
    \begin{split}
	\mathcal{L}(\alpha, \beta, \gamma, \delta \mid X) = \dfrac{n!}{n_1! \dots n_6!} (\gamma a)^{n_1} \\ (\beta (1-a))^{n_2} (\alpha a)^{n_3} \\ (\delta (1-a))^{n_4} ((1-\alpha-\gamma)a)^{n_5} \\ ((1-\beta-\delta)(1-a))^{n_6}
	\end{split}
\end{equation*}
and consequently the log-likelihood function is
\begin{gather}
    \begin{split}
	\log \mathcal{L}(\alpha, \beta, \gamma, \delta \mid X) = \log\left(\dfrac{n!}{n_1! \dots n_6!}\right) &+ \\ n_1 \log(\gamma a) +  n_2\log(\beta (1-a)) &+ \\ n_3\log(\alpha a) 
	+ n_4 \log(\delta (1-a)) &+ \\ n_5\log((1-\alpha-\gamma)a) &+ \\ n_6\log((1-\beta-\delta)(1-a))
	\end{split}
\end{gather}

To obtain MLE estimates of parameters $\alpha, \beta, \gamma, \delta$, we wish to solve the following optimization problem
\begin{equation}\label{mle-opt}
	\max_{(\alpha, \beta, \gamma, \delta) \in [0, 1]^4} \log \mathcal{L}(\alpha, \beta, \gamma, \delta \mid X)
\end{equation}

By Fermat's theorem, the optimal solution to the above optimization problem lies at either a boundary point or a stationary point.

The boundary points of the set $[0, 1]^4$ are given by the set
\begin{equation*}
\begin{split}
    B := \{  (\alpha, \beta, \gamma, \delta) \in [0, 1]^4:  \alpha=0 \text{ or } \alpha=1 \text{ or } \\ \beta=0 \text{ or } \beta=1 \text{ or } \\ \gamma =0 \text{ or } \gamma =1 \text{ or } \\ \delta=0 \text{ or } \delta=1  \}
\end{split}
\end{equation*}
The value of the function $\log \mathcal{L}(\alpha, \beta, \gamma, \delta)$ is negatively unbounded on the set $B$. 

On the other hand, the stationary points can be determined by setting the gradient to be zero, i.e., by solving the equation
\begin{equation*}
\nabla_{(\alpha, \beta, \gamma, \delta)} \log \mathcal{L}(\alpha, \beta, \gamma, \delta \mid X) = 0.
\end{equation*}

Solving the above equation yields the stationary point $(\alpha^*, \beta^*, \gamma^*, \delta^*)$ given as
\begin{align*}
	\alpha^* &= n_3/(n_1 + n_3 + n_5) \\
	\beta^* &= n_2/(n_2 + n_4 + n_6) \\
	\gamma^* &= n_1/(n_1 + n_3 + n_5) \\
	\delta^* &= n_4/(n_2 + n_4 + n_6)
\end{align*}
The value of the log-likelihood function at the above critical point, i.e., $\log \mathcal{L}(\alpha^*, \beta^*, \gamma^*, \delta^* \mid X)$ is positive which suggests that it is the optimizer of the optimization problem in (\ref{mle-opt}).
\end{proof}

\section{Details on MTurk study}
\label{appendix:mturk_study}
Figure~\ref{fig:mturk_instructions} shows the instruction page used to guide the workers on Mturk. Since our goal here was to simulate real user feedback for AI systems, we designed the prompt to mimic a situation where users provide their judgements on the accuracy of machine predictions on a given task. Each assignment page had 40 questions for the 20news task (50 for sst2), with an example question shown in Figure~\ref{fig:mturk_example}. For a real world application of this setting, we can imagine an email categorization model deployed on end-user email clients which automatically classifies incoming emails to a predefined class. The user would approve (select Accurate to above question) if the categorization was correct, reject or make correction (select Inaccurate to the question) or take no action. This closely follows the federated user feedback scenario described in our experiments with users explicitly providing positive or negative feedback. 

We recruited highly reliable annotators on Mturk by selecting past approval rate as $95\%+$ and number of past approved tasks as $5000+$. The average time for each task was 30 minutes, and the annotators were paid $\$7$ for completing the task which is above the US federal minimum hourly wage, given the average time for task completion. Note that we did not place any geographic restrictions on the annotators, nor reject any partial submissions, despite stating as such in the instruction sheet, as they were few in number.  

\begin{figure*}
    \centering
    \includegraphics[trim={4cm 5cm 2cm 1cm},width=1.15\linewidth]{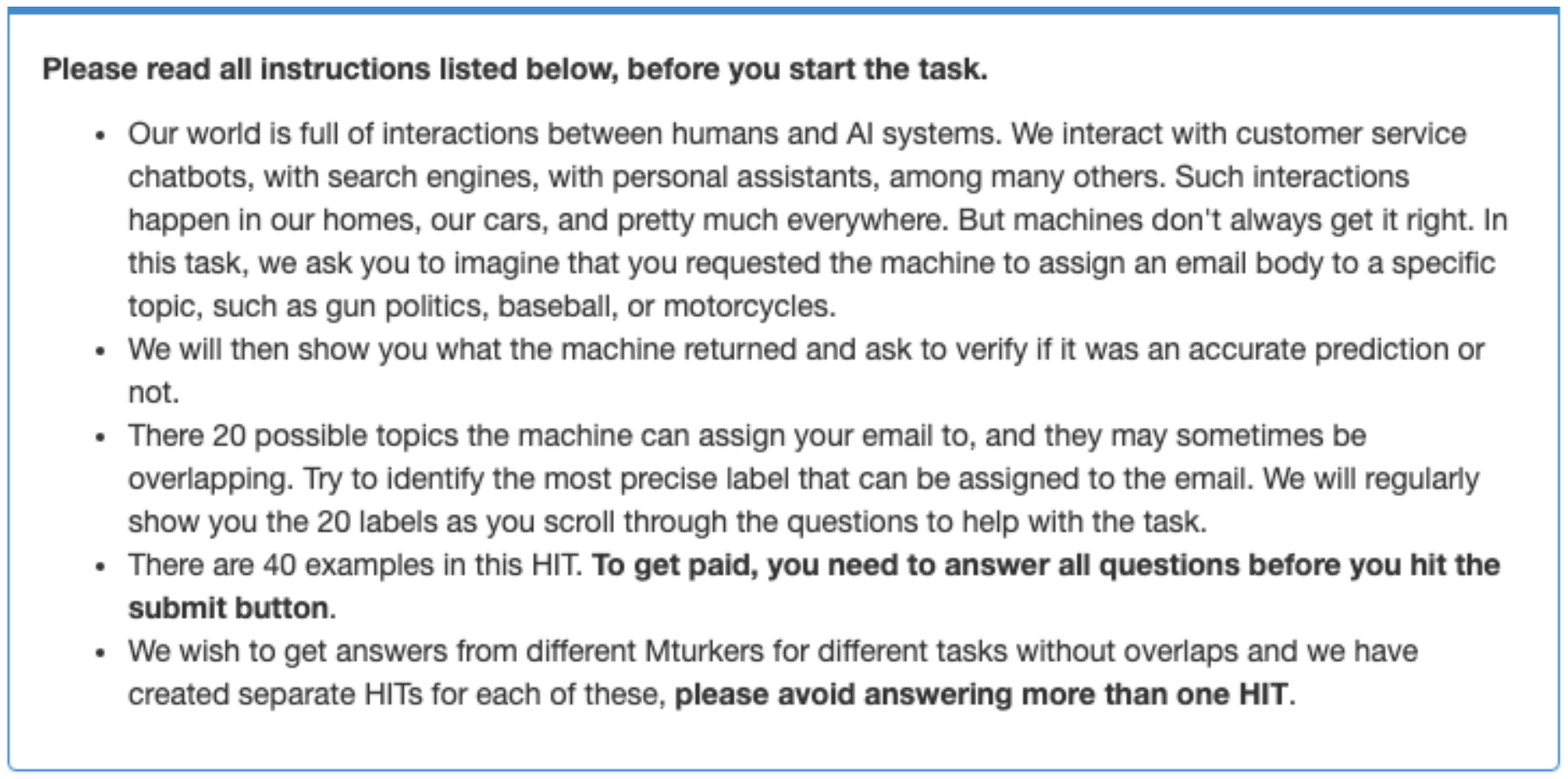}
    \caption{Instruction page with guidelines for Mturk annotators.}
    \label{fig:mturk_instructions}
\end{figure*}

\begin{figure*}
    \centering
    \includegraphics[trim={2cm 2cm 2.5cm 2cm},width=0.75\linewidth]{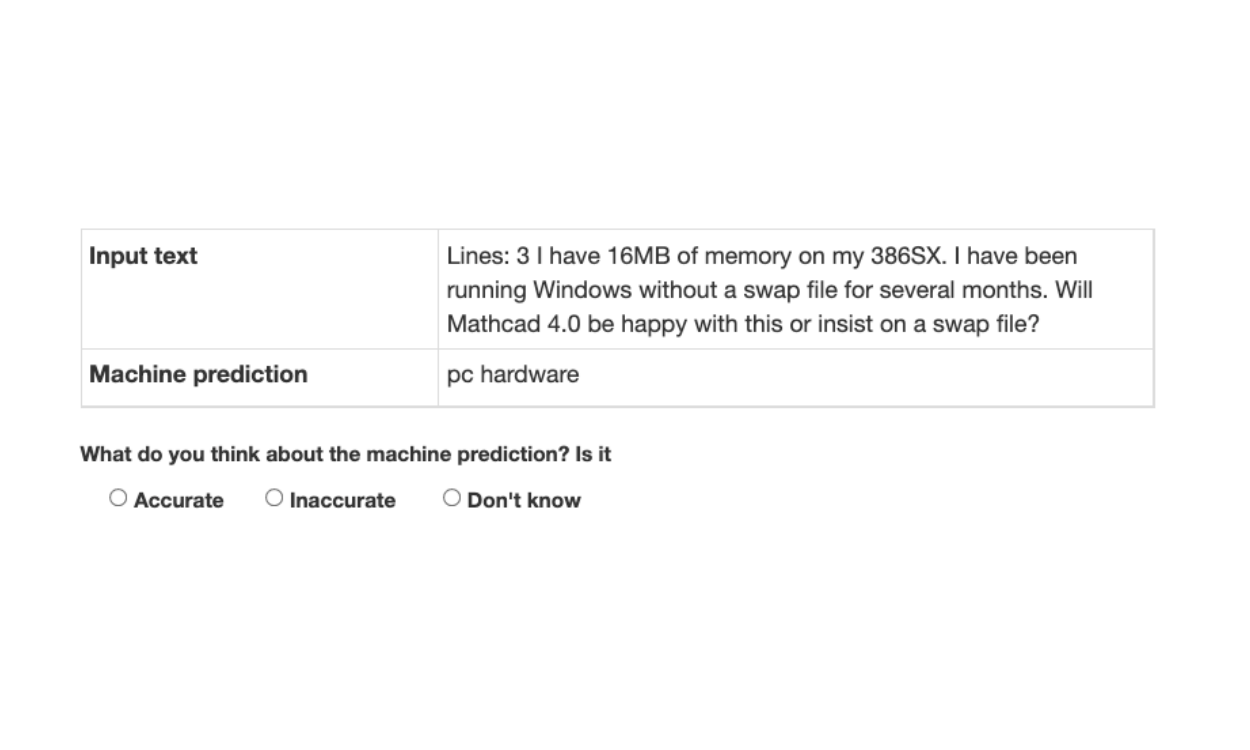}
    \caption{Example annotation task}
    \label{fig:mturk_example}
\end{figure*}

In Figure~\ref{fig:mturk_data}, we show the error in user feedback computed against gold labels for all the users. We also show the distribution of positive and negative responses for all the users. As evident from the figure, users provide positive feedback in majority cases. This behavior is expected since the initial model's accuracy for \emph{20news} is $59.14\%$ and for \emph{sst2} is $77.37\%$; since a the majority of the pseudo labels are correct predictions, we expect mostly positive feedback from the users. 

\begin{figure*}
    \centering
    \includegraphics[width=\linewidth]{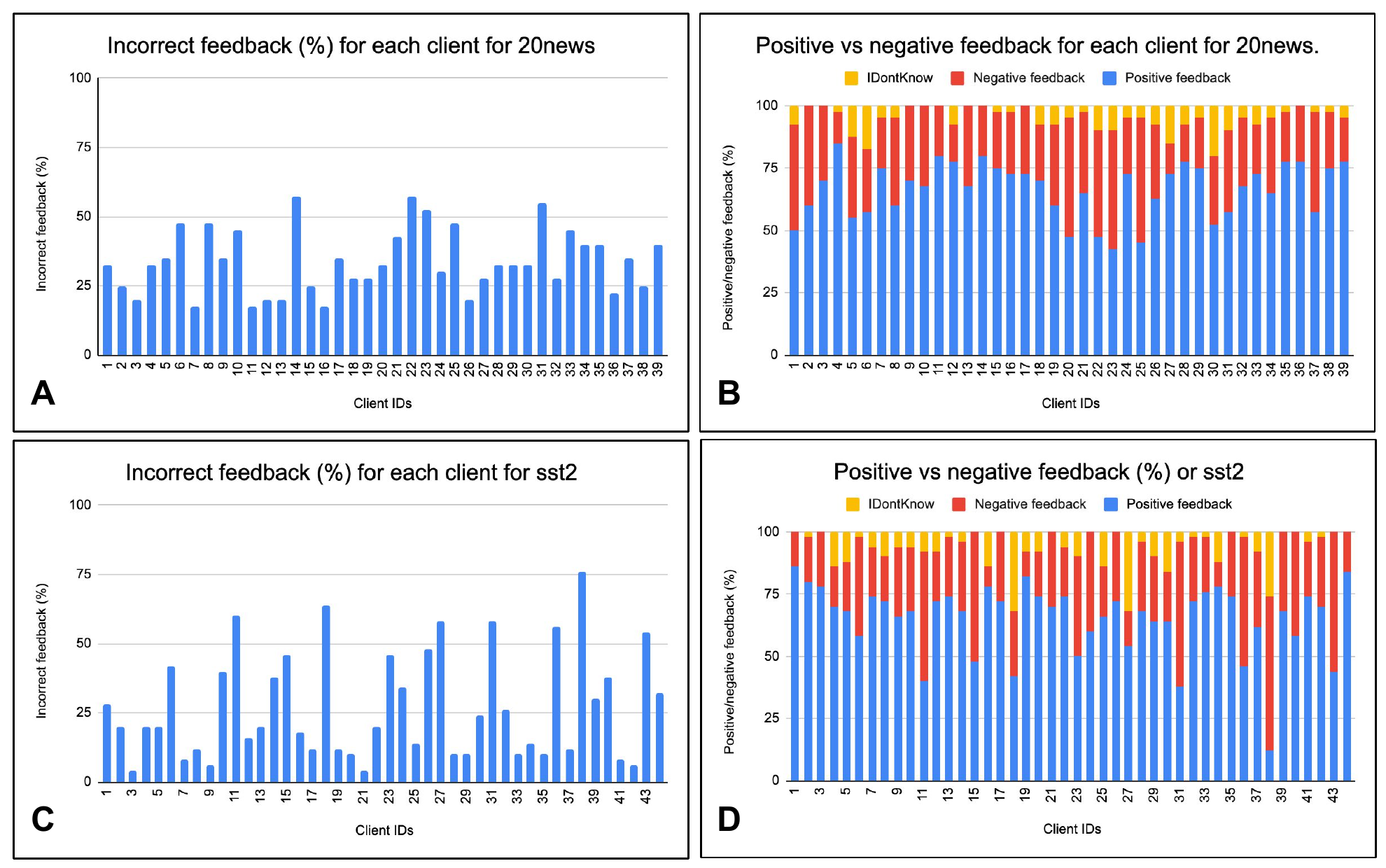}
    \caption{User feedback behavior of clients in Mturk case study. A \& C: Incorrect feedback(\%) for all the clients for 20news and sst2. B \& D: Distribution of negative and positive feedback for each client in 20news and sst2}
    \label{fig:mturk_data}
\end{figure*}
\end{document}